\documentclass[letterpaper, 10 pt, conference]{ieeeconf}
\IEEEoverridecommandlockouts
\usepackage[T1]{fontenc}
\usepackage[utf8]{inputenc}
\usepackage{cite}
\usepackage{amsmath,amssymb,amsfonts}
\usepackage{microtype}
\usepackage{array}
\usepackage{makecell} 
\usepackage{graphicx}
\usepackage{textcomp}
\usepackage{booktabs}
\usepackage{bbding}
\usepackage{tikz}
\usepackage{bm}
\usepackage[ruled,linesnumbered,noline]{algorithm2e}

\newtheorem{lemma}{Lemma}
\newtheorem{theorem}{Theorem}

\newtheorem{problem}{Problem}

\makeatletter
\def\proof{\@ifnextchar[{\@customproof}{\@customproof[Proof]}}
\def\@customproof[#1]{\par\noindent{\itshape #1. }\ignorespaces}

\makeatother

\makeatletter
\def\bstctlcite{\@ifnextchar[{\@bstctlcite}{\@bstctlcite[@auxout]}}
\def\@bstctlcite[#1]#2{\@bsphack
  \@for\@citeb:=#2\do{%
    \edef\@citeb{\expandafter\@firstofone\@citeb}%
    \if@filesw\immediate\write\csname #1\endcsname{\string\citation{\@citeb}}\fi}%
  \@esphack}
\makeatother

\def\BibTeX{{\rm B\kern-.05em{\sc i\kern-.025em b}\kern-.08em
T\kern-.1667em\lower.7ex\hbox{E}\kern-.125emX}}
\begin{document}
\bstctlcite{IEEEtranBSTCTL}

\title{\LARGE \bf
Reconfiguration-Complete Motion Primitives with Constructive Planning for Deformable Planar Modular Robots
}

\author{Jie Gu$^{1}$, Tingting Wang$^{1}$, Hongrun Gao$^{1}$, 
Yirun Sun$^{1}$, Zhihao Xia$^{1}$, Chunxu Tian$^{1}$, Dan Zhang$^{2}$%
\thanks{This work was supported by the National Natural Science Foundation 
of China under Grant 52305012. 
\textit{Jie Gu and Tingting Wang are co-first authors.} 
(Corresponding authors: Chunxu Tian; Dan Zhang.)}%
\thanks{$^{1}$Jie Gu, Tingting Wang, Hongrun Gao, Zhihao Xia, Yirun Sun, 
and Chunxu Tian are with the Institute of AI and Robotics, Academy for 
Engineering \& Technology, Fudan University, Shanghai 200433, China 
{\tt\footnotesize chxtian@fudan.edu.cn}.}%
\thanks{$^{2}$Dan Zhang is with the Department of Mechanical Engineering, 
The Hong Kong Polytechnic University, Hung Hom, Hong Kong SAR, China 
{\tt\footnotesize dan.zhang@polyu.edu.hk}.}%
\thanks{Digital Object Identifier (DOI): see top of this page.}%
}

\maketitle
\thispagestyle{empty}
\pagestyle{empty}

\begin{abstract}
The continuously deformable geometry of modular robots makes it difficult to define a fixed representation for reconfiguration planning and analysis. This letter introduces a square-cell abstraction that maps deformable rhombus modules to fixed-size grid cells while retaining physically interpretable local motions through two primitives, pivoting and shearing. Under this abstraction, we prove that every non-straight edge-connected configuration with $N \geq 7$ can be transformed to a fixed canonical staircase using only admissible primitive motions. Since these motions are reversible, any two configurations in this class are mutually reconfigurable. The proof is constructive and directly yields a staircase-canonicalization planner that transports removable boundary modules while preserving connectivity. As a practical enhancement, we further introduce a boundary-to-delivery lookahead selector that ranks admissible high level choices without affecting the completeness guarantee. Experiments demonstrate the constructive reconfiguration process and show that the selector substantially reduces planning time, while reference comparisons indicate lower planning times than the prior framework over the shared module counts.
\end{abstract}

\begin{figure*}
    \centering
    \includegraphics[width=1\linewidth]{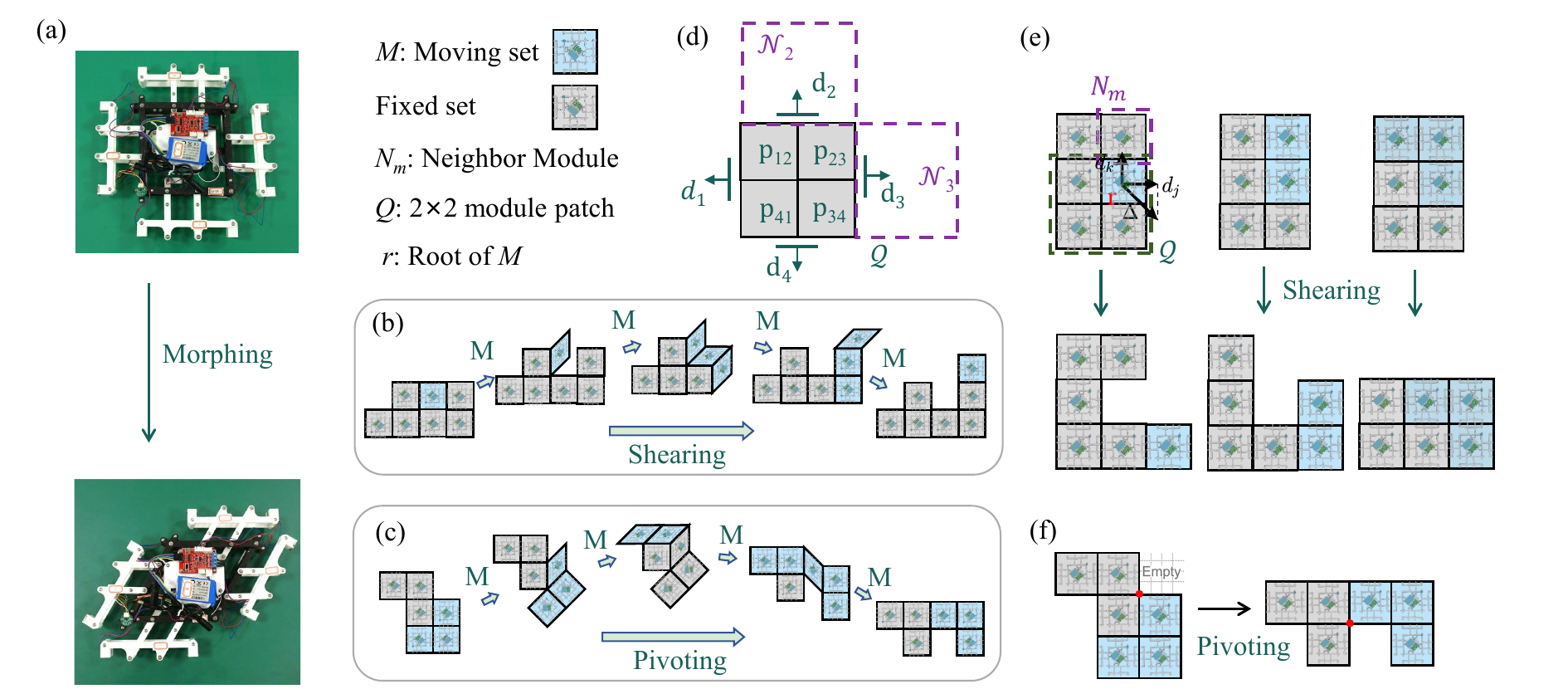}
    \caption{Continuous morphing and the corresponding primitive-motion abstractions. (a) Example morphing of the deformable rhombus platform. (b) Representative shearing sequence. (c) Representative pivoting sequence. (d) Geometric notation for a fully occupied $2\times2$ patch and its neighboring regions. (e) Shearing primitive and neighbor condition. (f) Pivoting primitive about a lattice vertex.}
    \label{fig:pivot and shear}
\end{figure*}

\section{Introduction}

Modular self-reconfigurable robots (MSRRs) are composed of repeated robotic units~\cite{castano2000conro,rus2001crystalline} that can alter their morphology, connectivity, and functional arrangement through reconfiguration~\cite{fukuda1988approach,yim2007modular}. Over the past decades, a broad range of MSRR architectures has been developed, reflecting different tradeoffs in mobility~\cite{romanishin2013mblocks,romanishin2015mblocks}, connectivity, and structural versatility~\cite{salemi2006superbot,ahmadzadeh2016modular,vu2023modular}. Despite this mechanical diversity, a common challenge is to synthesize globally feasible reconfiguration sequences from local connection and motion primitives~\cite{liang2026modular}.

Planar deformable modules make this challenge particularly pronounced because their geometry changes continuously during reconfiguration~\cite{lyder2008odin}, complicating both state representation and motion planning. Shape-changing triangular modules have been explored for planar structural reconfiguration~\cite{gerbl2024parts}, while rhombus-based modules exploit continuous geometric deformation to realize stable reconfiguration motions~\cite{gu2026rhomorph}. Practical execution further depends on docking, alignment, and local motion feasibility, which must be respected by the planner~\cite{feng2024autonomous}. For the deformable rhombus modules considered here, this creates a fundamental representation problem: planning and reconfigurability analysis require a fixed discrete state description that retains the connectivity and local motion constraints relevant to physical execution without explicitly modeling every continuous deformation degree of freedom.

Existing reconfiguration methods address different aspects of this problem. Early work formulated reconfiguration as a search over local module actions~\cite{casal1999self,rus1999compressible}, while complexity results show that optimal reconfiguration planning is hard in standard modular-robot settings~\cite{hou2010complexity,ye2019np}. Search-based and engineering-oriented methods can handle practical instances but may rely on heuristics, auxiliary workspace, or support structures~\cite{thalamy2019survey,thalamy2021engineering}. Meta-module approaches simplify planning by grouping modules into higher-level units~\cite{brandt2007new,dewey2008generalizing}, whereas theoretical treatments of lattice reconfiguration establish strong reachability properties under abstract local transformations~\cite{aloupis2013efficient,bubeck2008pushing-hypercubes,hurtado2015distributed}. Such abstractions are valuable for global analysis, but their transformations do not necessarily correspond directly to the local motions available to a continuously deformable planar module.

This leads to the central question considered in this work: can continuously deformable modules be mapped to a fixed discrete representation that retains physically interpretable local motions while still supporting a global reconfigurability guarantee? A recent framework provides initial-to-target planning for deformable quadrilateral modular robots~\cite{gu2026selfreconfiguration}. Complementing this planning framework, we focus on the representation-and-reachability problem: whether a compact discrete abstraction together with a fixed set of local motion primitives is sufficient to establish global reconfigurability for deformable rhombus modules.

We address this question using a square-cell abstraction in which each deformable rhombus module is represented by its square pose on the integer lattice. The abstraction retains lattice occupancy, edge adjacency, exterior-boundary access, and the local feasibility constraints required for reconfiguration, while removing continuous rhombus-angle variables from the global state description. Within this representation, we define two admissible primitive classes, pivoting and shearing, illustrated in Fig.~\ref{fig:pivot and shear}. We then prove that every non-straight edge-connected configuration with $N\ge 7$ can be transformed, using only these primitives, to a fixed canonical staircase $S_N$. Since the admissible primitive motions are reversible, the common canonical form establishes mutual reconfigurability between any two configurations in this class. The resulting completeness guarantee applies to the square-cell abstraction; its connection to continuous rhombus-module execution is captured through the primitive-motion templates and associated feasibility assumptions, while closed-loop hardware validation is left for future work.

The proof is constructive and directly yields a staircase-canonicalization planner. The planner maintains a decomposition into a staircase prefix, a remainder, and a mobile gadget, and repeatedly harvests removable boundary modules, transports them along the exterior boundary, and places them at the staircase tip while preserving edge connectivity. Thus, the planning procedure is not separate from the completeness proof, but follows directly from its constructive argument. As a practical enhancement to this planner, we further introduce a boundary-to-delivery lookahead selector that ranks admissible boundary targets and post-capture gadget states to reduce unnecessary transport and local search. The selector changes only the ordering of admissible high level choices and therefore does not alter the primitive set or the completeness guarantee.

The main contributions are as follows.
\begin{enumerate}
    \item We formulate a square-cell abstraction for planar deformable rhombus modular robots and define admissible pivoting and shearing primitives with explicit feasibility conditions.

    \item We prove that every non-straight edge-connected configuration with $N\ge 7$ can be transformed to a canonical staircase, establishing universal reconfigurability under the proposed abstraction.

    \item We develop a constructive staircase-canonicalization planner from the proof and further introduce a boundary-to-delivery lookahead selector to improve planning efficiency without affecting completeness.
\end{enumerate}

\section{Definition}
\subsection{Basic model}
The motivating system is a two-dimensional modular robot composed of identical deformable rhombus modules. For mathematical analysis and constructive planning, we use a square-cell abstraction: each physical rhombus module is represented by one occupied unit square whose center lies on the integer lattice. A configuration therefore records the occupied cells, rather than module identities, as a finite set of integer coordinates
\begin{equation}
C = \{ \mathbf{p}_1, \mathbf{p}_2, \dots, \mathbf{p}_N \} \subset \mathbb{Z}^2,
\end{equation}
where each $\mathbf{p}_i = (x_i,y_i)$ denotes the center of an abstract cell. This abstraction preserves the combinatorial contact structure and the local motion constraints used by the planner, while removing continuous rhombus-angle variables from the proof. For any lattice vertex $v \in \mathcal{V}$, where $\mathcal{V}:=\mathbb{Z}^2 + (\tfrac12,\tfrac12)$, we define the vertex density $D(v)\in\{0,1,2,3,4\}$ as the number of modules incident to the vertex $v$.

Two modules centered at $\mathbf{p},\mathbf{q}\in C$ are edge-adjacent if and only if $\|\mathbf{p}-\mathbf{q}\|_1 = 1$, where $\|\cdot\|_1$ denotes the
$\ell_1$ norm on $\mathbb{Z}^2$.
We collect these adjacencies in the induced graph $G(C)=(C,E)$, with
\begin{equation}
    E := \bigl\{ \{\mathbf{p},\mathbf{q}\} \subseteq C \mid \|\mathbf{p}-\mathbf{q}\|_1 = 1 \bigr\},
\end{equation}
A configuration is edge-connected if $G(C)$ is connected, or equivalently if any two modules are joined by an edge-adjacent path. This edge connectivity is the global connectivity constraint maintained by every admissible transition below.

To describe reconfiguration, we define a generic primitive motion as acting on a subset of modules, denoted as the moving set $M \subseteq C$. Let $S := C \setminus M$ denote the stationary set. Given a transformation map $f: \mathbb{Z}^2 \to \mathbb{Z}^2$ associated with the specific motion, the target of the moving set is $M' := \{ f(\mathbf{p}) \mid \mathbf{p} \in M \}$, and the resulting configuration is $C' := S \cup M'$. A transition is topologically admissible if and only if the following common constraints are met:
\begin{enumerate}
    \item $M\neq\varnothing$, $S\neq\varnothing$, and both $M$ and $S$ are edge-connected.
    \item The target cells are unoccupied by the stationary modules, i.e., $M' \cap S = \varnothing$.
    \item The final configuration $C'$ is edge-connected.
\end{enumerate}
\subsection{Pivoting}
As illustrated in Fig.~\ref{fig:pivot and shear}(f), pivoting is a local primitive motion in which the moving set $M$ rotates as a rigid body by $90^\circ$ about a lattice vertex of a module in $M$, called the pivot point $u \in\mathcal{V}$. We define the rotation map $f_p: \mathbb{Z}^2 \to \mathbb{Z}^2$ as a
$90^\circ$ counterclockwise rotation about $u$:
\begin{equation}
\label{eq:pivot-map}
f_p(\mathbf{p}) := \mathbf{p}_u + \mathbf{R}(\mathbf{p}-\mathbf{p}_u),
\quad \text{where } \mathbf{R} =
\begin{bmatrix} 0 & -1 \\ 1 & 0 \end{bmatrix}.
\end{equation}
Here $\mathbf{p}_u$ denotes the coordinate vector of the pivot vertex $u$; the clockwise case is obtained by using $\mathbf{R}^{-1}$.

A pivot is admissible if it is topologically admissible and the pivot point satisfies the vertex-density condition $D(u)=3$. Every admissible pivot is reversible by applying the opposite $90^\circ$ rotation about the same pivot point.

\subsection{Shearing}
To mitigate local deadlocks, a reversible primitive motion, called shearing and illustrated in Fig.~\ref{fig:pivot and shear}(e), is defined over a fully occupied $2\times2$ module patch $\mathcal{Q}\subset C$ (Fig.~\ref{fig:pivot and shear}(d)) and a neighboring module ${N}_m$ chosen from the neighbor set $\mathcal{N}_k$ defined below. A shear translates an edge-connected moving set $M \subset C$ rooted in a chosen corner of the patch $\mathcal{Q}$.

For this patch, let $\mathcal{D} = \{\mathbf{d}_1, \mathbf{d}_2, \mathbf{d}_3, \mathbf{d}_4\}$ be the set of four outward, grid-aligned unit normal vectors of $\mathcal{Q}$, indexed in counterclockwise cyclic order. We denote the centers of the four corner cells by:
\begin{equation}
    \mathcal{Q} = \{ \mathbf{p}_{12}, \mathbf{p}_{23}, \mathbf{p}_{34}, \mathbf{p}_{41} \},
\end{equation}
where $\mathbf{p}_{kj}$ denotes the center of the unique corner cell adjacent to the two sides with outward normals $\mathbf{d}_k$ and $\mathbf{d}_j$, and $\mathbf{p}_{kj}=\mathbf{p}_{jk}$.

For the selected side $\mathbf{d}_k$, define the occupied neighbor set
\begin{equation}
    \mathcal{N}_k(\mathcal{Q}) := \bigl\{ \mathbf{q}+\mathbf{d}_k \mid \mathbf{q}\in\mathcal{Q},\ \mathbf{q}+\mathbf{d}_k \in C\setminus\mathcal{Q} \bigr\}.
\end{equation}

Given these labels, fix a normal index $k \in \{1, \dots, 4\}$ and choose an adjacent index $j \in \{k-1, k+1\}$, with indices taken modulo 4. The root is defined as $\mathbf{r} := \mathbf{p}_{kj} \in \mathcal{Q}$ with the constraint
\begin{equation}
    M\cap\mathcal{Q}=\{\mathbf{r}\}.
\end{equation}

The shear map $f_s: \mathbb{Z}^2 \to \mathbb{Z}^2$ is defined as a translation of the moving set by the difference of the chosen normal vectors:
\begin{equation}
    \label{eq:shear-map}
    f_s(\mathbf{p}) := \mathbf{p} + \boldsymbol{\Delta}, 
    \quad \text{where } \boldsymbol{\Delta} := \mathbf{d}_j - \mathbf{d}_k.
\end{equation}

In addition to topological admissibility, a shear requires the geometric neighbor condition: the $2\times2$ patch must have at least one occupied neighbor on its $\mathbf{d}_k$ side. Equivalently, this means $\mathcal{N}_k(\mathcal{Q})\neq\varnothing$, and the chosen neighbor satisfies $\mathcal{N}_m\in\mathcal{N}_k(\mathcal{Q})$.

Every admissible shear is reversible by applying the operation with displacement $-\boldsymbol{\Delta}$ (swapping the roles of $\mathbf{d}_k$ and $\mathbf{d}_j$) in the configuration $C'$.

\section{Problem Statement and Main Theorem}
\label{sec:problem-main}

Given the admissible primitive motions in Section~II, a \emph{reconfiguration sequence} is a finite sequence of configurations
$C_0,C_1,\dots,C_T$ such that each $C_{t+1}$ is obtained from $C_t$ by one admissible primitive motion. The central question is whether such a sequence exists between two configurations and, if so, how to construct it. This leads to Problem~\ref{prob:universal}:

\begin{problem}[Universal Reconfiguration]
\label{prob:universal}
Given two edge-connected configurations $C,C'\subset\mathbb{Z}^2$ of the same size $N$, decide whether $C$ can be reconfigured into $C'$ under admissible primitive motions, and if so, construct an explicit plan.
\end{problem}

Straight-line configurations are isolated under the primitives of Section~II and are excluded throughout. Small cases with $N\le 6$ can be handled separately by finite enumeration, so the main result focuses on the regime $N\ge 7$.

\begin{theorem}[Universal Reconfigurability]
\label{thm:universal}
For any two non-straight edge-connected configurations $C$ and $C'$ of size $N\ge 7$, there exists a finite sequence of admissible primitive motions that transforms $C$ into $C'$.
\end{theorem}

The proof is constructive and is organized around a fixed canonical staircase $S_N$ of size $N$. We define $S_N$ directly as a staircase made of $N$ unit cells, anchored at the origin and generated by alternating horizontal and vertical steps.
Let
$s_1=(0,0)$ and for $i\ge 1$ define
\begin{equation}
s_{i+1} \;=\;
\begin{cases}
s_i + (1,0), & i \ \text{odd},\\[2pt]
s_i + (0,1), & i \ \text{even}.
\end{cases}
\end{equation}
Equivalently, $s_i=\bigl(\lceil\frac{i-1}{2}\rceil,\ \lfloor\frac{i-1}{2}\rfloor\bigr)$. Then the staircase configuration is
\begin{equation}
S_N \;=\; \{\, s_i \in \mathbb{Z}_{\ge 0}^2 \mid 1\le i \le N \,\}.
\end{equation}

\begin{theorem}[Reachability to Canonical Staircase]
\label{thm:reachability-staircase}
For any non-straight edge-connected configuration $C$ of size $N\ge 7$, there exists a finite sequence of admissible primitive motions transforming $C$ into the canonical staircase $S_N$.
\end{theorem}

\begin{proof}[Proof of Theorem~\ref{thm:universal}]
Apply Theorem~\ref{thm:reachability-staircase} to obtain a plan $C\to S_N$. Applying the same theorem to $C'$ gives $C'\to S_N$; reversing that plan using reversibility of the primitive motions yields $S_N\to C'$. Concatenating the two sequences gives $C\to S_N\to C'$.
\end{proof}

\section{Proofs of Reconfigurability}
\label{sec:canonicalization-alg}

\begin{algorithm}[t]
\caption{\textsc{StaircaseCanonicalize}$(C)$}
\label{alg:canonicalize}
\KwIn{Edge-connected configuration $C$ of size $N\ge 7$ that is not a straight line.}
\KwOut{A plan transforming $C$ into $S_N$.}

\BlankLine
\textbf{Initialize:} \textsc{Initialize} $C_0$ \;

\While{$R\neq\varnothing$}{
  \textsc{Select}($x$) on the outer boundary$\partial B$ \;
  \textsc{BoundaryFlow}$(G^\circ) \leadsto x$ and dock next to $x$ \;
  \textsc{Pickup}$(x)$ to obtain $\widetilde{G}^\bullet$ and update $R\leftarrow R\setminus\{x\}$ \;
  \textsc{GadgetNormalize}$(\widetilde{G}^\bullet) \leadsto G^\bullet$ \;
  \textsc{BoundaryFlow}$(G^\bullet) \leadsto \mathrm{tip}(P_k)$ and dock at the tip \;
  \eIf{$R=\varnothing$}{
    \textsc{CloseOut}$(G^\bullet)$ to merge the gadget into the prefix and obtain $S_N$ \;
  }{
    \textsc{PlaceAndExtend}$(G^\bullet) \leadsto P_{k+1}$ and $G^\circ$ \;
  }
}
\end{algorithm}

To prove Theorem~\ref{thm:reachability-staircase}, Algorithm~\ref{alg:canonicalize} gives a constructive canonicalization planner. It maintains a staircase prefix, a remainder, and a mobile gadget, and repeatedly performs one harvest, transport, and place cycle: select a removable boundary module, move the empty gadget to it, absorb it into the carrying gadget, return to the staircase tip, and extend the prefix.

The remainder of this section follows the proof structure behind the algorithm. Section~IV-A defines the maintained decomposition and boundary notation, Section~IV-B states the lemmas that make each high level step feasible, and Section~IV-C maps the algorithmic operations to primitive motion sequences or bookkeeping updates.

\subsection{Configuration Decomposition}
\label{sec:decomposition}

At each stage, the planner maintains the disjoint decomposition
\begin{equation}
C \;=\; P_k \uplus R \uplus G^\sigma,
\label{eq:decomposition}
\end{equation}
where $P_k$ is the length-$k$ staircase \emph{prefix} already fixed in the canonical form, $R$ is the \emph{remainder} still to be harvested, and $G^\sigma$ is the mobile gadget. The superscript $\sigma\in\{\circ,\bullet\}$ denotes whether the gadget is empty or carrying one payload module. We use two standard gadget shapes:
\begin{itemize}
\item $G^\circ$ (empty), with $|G^\circ|=5$;
\item $G^\bullet$ (carrying one payload module), with $|G^\bullet|=6$.
\end{itemize}
The base $B := P_k \uplus R$ contains all modules that are not currently part of the mobile gadget. Let $\mathcal{W}_\infty$ be the unique unbounded connected component of the complement lattice $\mathbb{Z}^2 \setminus B$. The outer boundary used by the gadget is the subset of base modules adjacent to this unbounded exterior:
\begin{equation}
\partial B = \{ \mathbf{p} \in B \mid \exists \mathbf{q} \in \mathcal{W}_\infty \text{ s.t. } \mathbf{p}, \mathbf{q} \text{ are edge-adjacent} \}.
\end{equation}
Hence any selected harvest target $x\in R\cap\partial B$ is accessible from the exterior boundary traversed by the gadget. Fig.~\ref{fig:configuration_decomposition} illustrates the roles of $P_k$, $R$, and $G^\sigma$.
We repeatedly use three invariants: $P_k$ is a staircase prefix and remains fixed, $B=P_k\uplus R$ remains edge-connected, and the gadget is docked on the outer boundary component at the beginning and end of each high level operation. These are the main properties used by the lemmas below.

\begin{figure}[t]
    \centering
    \includegraphics[width=\linewidth]{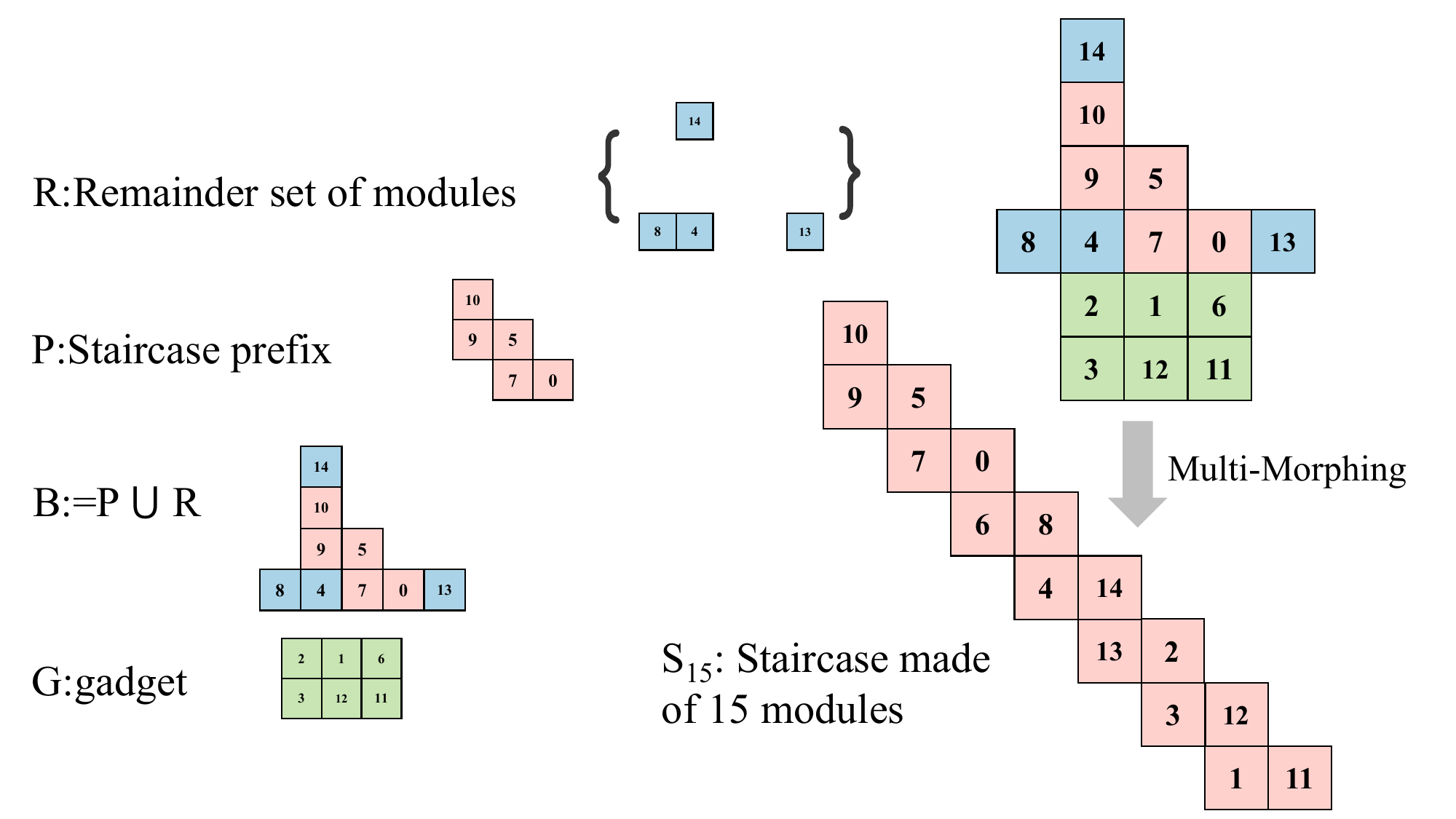}
    
    \caption{Illustration of the configuration decomposition. The overall configuration is disjointly partitioned into a staircase prefix ($P$), the remainder ($R$), and a mobile gadget ($G$) for module transportation.}
    
    \label{fig:configuration_decomposition}
\end{figure}

\subsection{Core Lemmas}
\label{subsec:core-lemmas}
The following lemmas provide the key claims needed to justify Algorithm~\ref{alg:canonicalize}. Lemma~\ref{lem:initializability} creates the initial decomposition; Lemmas~\ref{lem:existence}--\ref{lem:traversability} ensure that each iteration can select and reach a removable outer-boundary target; Lemma~\ref{lem:carry-normalize} restores the carrying-gadget shape after pickup; and Lemma~\ref{lem:convergence} handles the final closeout. Together, these lemmas show that each cycle is feasible and decreases $|R|$ by one.

\begin{lemma}[Initializability]
\label{lem:initializability}
For any configuration $C$ with $N\ge 7$ that is not a straight line, there exists a finite primitive sequence that reaches an initial decomposed configuration
\begin{equation}
C_0 = P_{k_0}\uplus R_0\uplus G^\circ,
\label{eq:init_decomp}
\end{equation}
where $P_{k_0}$ is an initial staircase prefix of length $k_0$, $R_0$ is the initial remainder, and $G^\circ$ is the standard empty gadget docked on the outer boundary component of the base $B_0=P_{k_0}\uplus R_0$.
\end{lemma}

\begin{proof}
Choose a boundary module $t \in \partial C$ and a connected
subset $U \subseteq C$ with $|U| = 7$ and $t \in U$ such that
$M := U \setminus \{t\}$ is edge-connected, not a straight line,
and $C \setminus M$ is edge-connected
(such a choice can be ensured by selecting $M$ as a non-bridge
six-module cluster on the outer boundary).
By Lemma~4, while keeping $t$ fixed we can locally reconfigure
$M$ into the standard carrying gadget $G^\bullet$ docked at $t$.
Let $x \in G^\bullet$ be a corner backbone module adjacent to $t$,
and define $G^\circ := G^\bullet \setminus \{x\}$.
Define the resulting base as
$B_0 := C \setminus G^\circ = (C \setminus M) \cup \{x\}$.
Since $x$ is adjacent to $t \in C \setminus M$, $B_0$ remains
edge-connected.
Within $B_0$, select as many modules as possible that align with the staircase pattern to form the initial prefix $P_{k_0}$; the remaining modules constitute $R_0$. Hence,
$C_0=P_{k_0}\uplus R_0\uplus G^\circ$ is constructible.
\end{proof}

\begin{figure}
    \centering
    \includegraphics[width=1\linewidth]{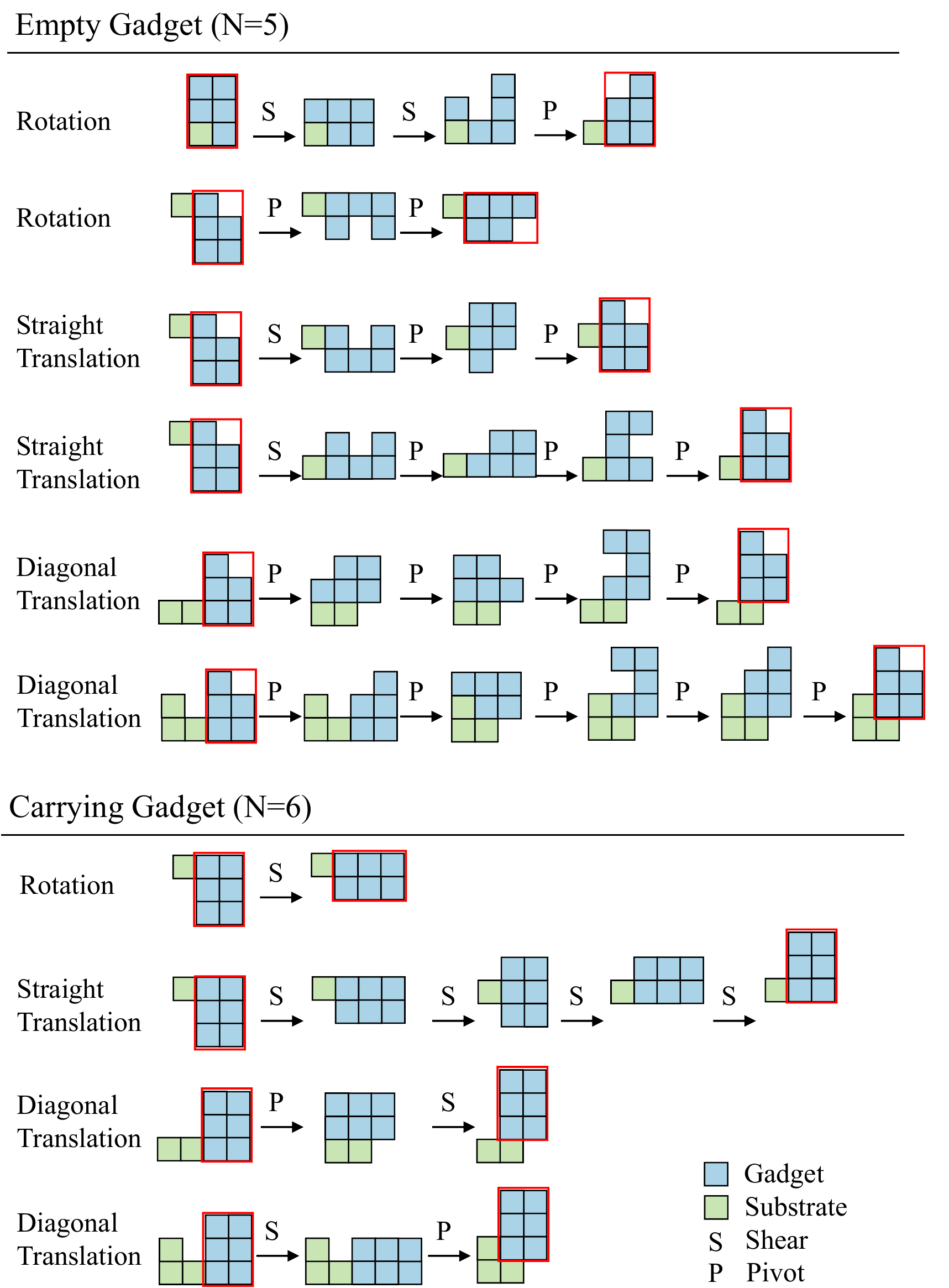}
    \caption{Primitive templates underlying the high level gadget motions used by the constructive planner.}
    \label{fig:metamotion}
\end{figure}

\begin{lemma}[Existence of a removable boundary module]
\label{lem:existence}
Let $B=P_k\uplus R$ with $R\neq\varnothing$. Then there exists a module
$x\in R\cap \partial B$ lying on the \emph{outer boundary component} of $B$
such that $B\setminus\{x\}$ remains edge-connected.
\end{lemma}
\begin{proof}
Let $G_B$ be the edge-adjacency graph of $B$ and let $\partial B$ be defined by adjacency to the unbounded exterior $\mathcal W_\infty$. Since $R\neq\varnothing$ and $P_k$ is a staircase path, $P_k$ cannot separate the plane; hence $R\cap\partial B\neq\varnothing$. Fix any $b\in R\cap\partial B$.

Consider the block-cut tree $T$ of $G_B$ rooted at the block containing $P_k$. Let $\mathcal L$ be the \emph{leaf block of $T$} contained in the subtree reached by following the unique root-to-$b$ path in $T$ (i.e., $\mathcal L$ is the last block on that path and has degree $1$ in $T$), and let $a$ be the unique cut vertex adjacent to $\mathcal L$ in $T$ (if $\mathcal L$ is the root block, take
$a=\varnothing$).
Because $b$ lies in this subtree, we have $\mathcal L\cap R\cap\partial B\neq\varnothing$.

If $a=\varnothing$, then $G_B$ has no articulation vertices and we may take
$x=b$. Otherwise, since $\mathcal L$ is a leaf block, the only vertex of
$\mathcal L$ that belongs to another block is $a$; thus every
$x\in\mathcal L\setminus\{a\}$ is not an articulation vertex of $G_B$.
Choose any
$x\in(\mathcal L\cap R\cap\partial B)\setminus\{a\}$.
Then $x$ is not an articulation vertex of $G_B$, hence $B\setminus\{x\}$ remains edge-connected.
\end{proof}

\begin{lemma}[Traversability]
\label{lem:traversability}
The standard gadget $G^\sigma$ can move along the \emph{outer boundary component} of the base $B$. Each step is a gadget \emph{meta-motion}, implemented by a sequence of admissible primitive motions.
\end{lemma}

\begin{proof}
Let $\mathcal W_\infty$ be the unbounded component of $\mathbb Z^2\setminus B$. We define a vacancy $e \in \mathcal W_\infty$ as a \emph{groove unit} if exactly three of its four edge-neighbors are occupied by modules of $B$; this implies $e$ has a unique neighbor $e' \in \mathcal W_\infty$. Recursively, if $e'$ has exactly two occupied neighbors in $B$ (forming a tunnel), we also classify it as a groove unit and continue extending the chain outward. The process terminates at the first vacancy possessing at least three vacant neighbors in $\mathcal W_\infty$. The \emph{groove depth}, denoted by $d_g$, is defined as the total number of groove units in this chain.

All auxiliary boundary adjustments used to resolve grooves act only on modules in $R$, leaving the staircase prefix $P_k$ unchanged.

If $d_g > 1$, we apply an anti-shear to transfer a module from the groove's side-wall into the groove. This operation reduces $d_g$ by at least $1$. We iterate this process until $d_g = 1$. If $d_g = 1$, the remaining groove is eliminated by repeatedly pivoting an adjacent boundary module into the vacancy. Each pivot maintains the edge-connectivity of $B$ and propagates the vacancy along the boundary contour. Since this motion effectively shifts the local concavity from the middle of a segment toward a convex endpoint, the groove disappears after a finite number of pivots. Thus, all grooves can be resolved locally using admissible primitives, rendering the outer boundary groove-free.

On a groove-free outer boundary, a gadget encounters only three local surface patterns: an in-place $90^\circ$ rotation, a unit axis-aligned translation, or a unit diagonal translation. For both $G^\circ$ and $G^\bullet$, Fig.~\ref{fig:metamotion} provides a meta-motion template realizing each pattern as a finite sequence of admissible primitives.

Therefore, by repeatedly instantiating these templates, $G^\sigma$ can be flowed along the outer boundary component of $B$ between any two valid docking locations.
\end{proof}

\begin{lemma}[Carrying-gadget normalization]
\label{lem:carry-normalize}
Any non-straight six-module gadget, with the dock fixed, can be locally reconfigured by admissible primitive motions into the standard $G^\bullet$, preserving connectivity and avoiding collisions.
\end{lemma}

\begin{proof}
Fix the dock and assume an exterior workspace $\mathcal W_{\mathrm{loc}}$ around the dock that is free of other modules, so that all intermediate motions of the six-module system are collision-free with the rest of the configuration. Let $\mathcal S$ be the set of all non-straight configurations of six modules with the dock fixed. Because the six modules remain edge-connected, all module positions lie within a finite lattice neighborhood of the dock; hence $\mathcal S$ is finite.

We exhaustively enumerate the transition graph on $\mathcal S$ under admissible primitives and verify that this graph is connected. Therefore, any two configurations in $\mathcal S$ are mutually reconfigurable within $\mathcal W_{\mathrm{loc}}$. In particular, any non-straight six-module gadget can be locally reconfigured into the standard carrying gadget $G^\bullet$ while preserving edge connectivity and avoiding collisions.
\end{proof}

\begin{figure}
    \centering
    \includegraphics[width=1\linewidth]{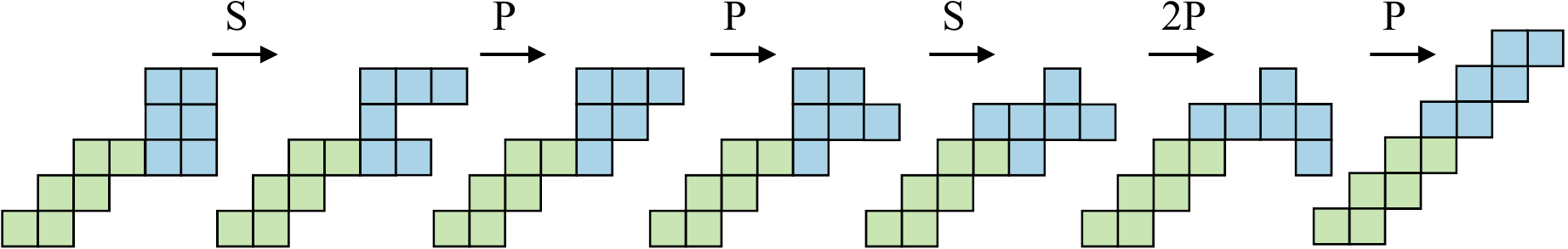}
    \caption{Illustration of the constructive proof for Lemma 5. A standard carrying gadget $G^\bullet$ is reconfigured into the final segment of the target staircase $S_N$ via a sequence of primitive motions.}
    \label{fig:convergence}
\end{figure}

\begin{lemma}[Convergence]
\label{lem:convergence}
When $R=\varnothing$, there exists a primitive motion sequence that merges the standard carrying gadget into the prefix, yielding the full staircase $S_N$.
\end{lemma}
\begin{proof}
Fig.~\ref{fig:convergence} gives a primitive motion template that fills the terminal cells of the target staircase $S_N$ using a docked standard carrying gadget $G^\bullet$, without collision and without breaking connectivity. Applying this template absorbs the gadget into the prefix and yields $S_N$.
\end{proof}

\subsection{High Level Operations}
\label{subsec:high_level_ops}

Algorithm~1 separates selection, motion, and bookkeeping. \textsc{Select} chooses a removable boundary target; \textsc{Pickup} and \textsc{PlaceAndExtend} update the decomposition; and \textsc{Initialize}, \textsc{BoundaryFlow}, \textsc{GadgetNormalize}, and \textsc{CloseOut} are primitive-motion routines.

\textbf{\textsc{Initialize:}}
Starting from a connected configuration $C$, this routine uses a finite primitive sequence to form the initial gadget-and-prefix decomposition $C_0$; Lemma~\ref{lem:initializability} guarantees feasibility.

\textbf{\textsc{Select:}}
This step picks a removable target module $x\in R\cap\partial B$ on the outer boundary of $B$; Lemma~\ref{lem:existence} guarantees that such an $x$ exists.

\textbf{\textsc{BoundaryFlow:}}
This routine moves the gadget $G^\sigma \in \{G^\circ, G^\bullet\}$ along the outer boundary to the selected docking location; Lemma~\ref{lem:traversability} guarantees feasibility, and Fig.~\ref{fig:metamotion} gives the motion templates.

\textbf{\textsc{Pickup:}}
Given $G^\circ$ docked next to $x$, this bookkeeping step updates the decomposition:
\begin{equation}
P_k \uplus R \uplus G^\circ \;\leadsto\; P_k \uplus (R \setminus \{x\}) \uplus \widetilde{G}^\bullet.   
\end{equation}

Only the labels change; the occupied cells do not.

\textbf{\textsc{GadgetNormalize:}}
After \textsc{Pickup}, Lemma~\ref{lem:carry-normalize} converts $\widetilde{G}^\bullet$ into the standard $G^\bullet$ by a finite local primitive sequence.

\textbf{\textsc{PlaceAndExtend:}}
If $R\neq\varnothing$, after
$G^\bullet$ has docked at the current tip
$t=\mathrm{tip}(P_k)$, the contacting corner module is relabeled
into the prefix, so this step is bookkeeping only.
\[
P_k \uplus R \uplus G^\bullet
\leadsto
P_{k+1} \uplus R \uplus G^\circ .
\]
This restores $G^\circ$ and extends the prefix.

\textbf{\textsc{CloseOut:}}
If $R=\varnothing$ after the final pickup,
the docked carrying gadget $G^\bullet$ is merged directly into
the prefix using the primitive sequence of Lemma~\ref{lem:convergence},
yielding $S_N$.

\begin{proof}[Proof of Theorem~\ref{thm:reachability-staircase}]
By Lemmas~1--5, each iteration decreases $|R|$ by one. Hence, after finitely many iterations, the final pickup makes $R=\varnothing$, and \textsc{CloseOut} yields $S_N$. Concatenating the finite admissible primitive sequences used in \textsc{Initialize}, \textsc{BoundaryFlow}, \textsc{GadgetNormalize}, and \textsc{CloseOut} yields a finite admissible primitive sequence transforming $C$ into $S_N$.
\end{proof}

\section{A Unified Lookahead Selector}
\label{sec:optimized-reconfiguration}
The constructive planner in Section~\ref{sec:canonicalization-alg} is complete, but it still leaves two coupled high level choices: which removable boundary module to harvest next and which reachable post-capture canonical state to continue from. We resolve both choices with one unified lookahead selector. The selector adds no new primitive motions; it only ranks admissible candidates and falls back to \textsc{Select} when the ranked set is empty.

At iteration $k$, let $W_k=(\omega_k^1,\ldots,\omega_k^h)$ be the next $h$ staircase sites generated by the baseline continuation rule, and let $A_k\subseteq R_k\cap\partial B_k$ be the removable modules on the outer boundary component. We partition $A_k$ into short contiguous boundary segments $Q_k=\{q_1,\ldots,q_{K_k}\}$ with $1\le |q_i|\le b$ and ignore hole boundaries, since harvesting stays on the unbounded exterior component. For each segment $q$, the baseline admissibility test selects a harvested module $x(q)\in q$ using the smallest boundary-distance proxy. After capturing $x(q)$, let
\begin{equation}
\mathcal{Z}_k(q)
:=
\{z\in\mathcal{S}_{\mathrm{can}} \mid \widetilde{G}_k^\bullet(q)\leadsto z\},
\label{eq:lookahead-candidates}
\end{equation}
where $\mathcal{S}_{\mathrm{can}}$ is the set of canonical carrying-gadget templates and $\leadsto$ denotes reachability by the local primitive search used for \textsc{GadgetNormalize}.

We score each admissible segment state pair $(q,z)$ by a compact scalar objective
\begin{equation}
\begin{aligned}
J(q,z;W_k)
&=
\lambda d_{\partial B}(\mathrm{dock}(G^\circ),q)\\
&\quad
+ L_{\mathrm{can}}(z)
+ \beta D_{\mathrm{del}}(z,\omega_k^1)\\
&\quad
- R_{\mathrm{lh}}(q;W_k,q_{k-1}),
\end{aligned}
\label{eq:unified-lookahead-score}
\end{equation}
where $d_{\partial B}$ is shortest-path distance along the outer boundary, $L_{\mathrm{can}}(z)$ is the number of local normalization moves to reach $z$, and $R_{\mathrm{lh}}(q;W_k,q_{k-1})$ is the short-horizon lookahead reward defined by
\begin{align}
A(q,W_k)=
\max_{\omega\in W_k}
\frac{1}{1+\min_{\mathbf p\in q}\|\mathbf p-\omega\|_1},
\label{eq:lookahead-alignment}\\
D_{\mathrm{del}}(z,\omega_k^1)
=
\min_{i\in\mathcal C(z)}
\|\mathbf p_i(z)-\omega_k^1\|_1,
\label{eq:lookahead-delivery}\\
H(q,q_{k-1})=
\begin{cases}
\frac{1}{1+d_{\partial B}(q,q_{k-1})}, & k>0,\\[2pt]
0, & k=0.
\end{cases}
\label{eq:unified-lookahead-terms}\\
R_{\mathrm{lh}}(q;W_k,q_{k-1})
=
\alpha A(q,W_k)+\nu H(q,q_{k-1}),
\label{eq:lookahead-reward}
\end{align}
Here $\mathcal C(z)$ is the set of endpoint modules in candidate state $z$. The weights $\lambda,\beta,\alpha,\nu>0$ are fixed tuning constants. The first three terms penalize transport, normalization, and delivery effort, while the internal weights $\alpha$ and $\nu$ balance the alignment and continuity subterms inside $R_{\mathrm{lh}}$.

\begin{equation}
(q_k^\star,z_k^\star)
\in
\arg\min_{q\in Q_k,\; z\in\mathcal Z_k(q)}
J(q,z;W_k),
\label{eq:unified-lookahead-selection}
\end{equation}
\begin{figure*}
    \centering
    \includegraphics[width=1\linewidth]{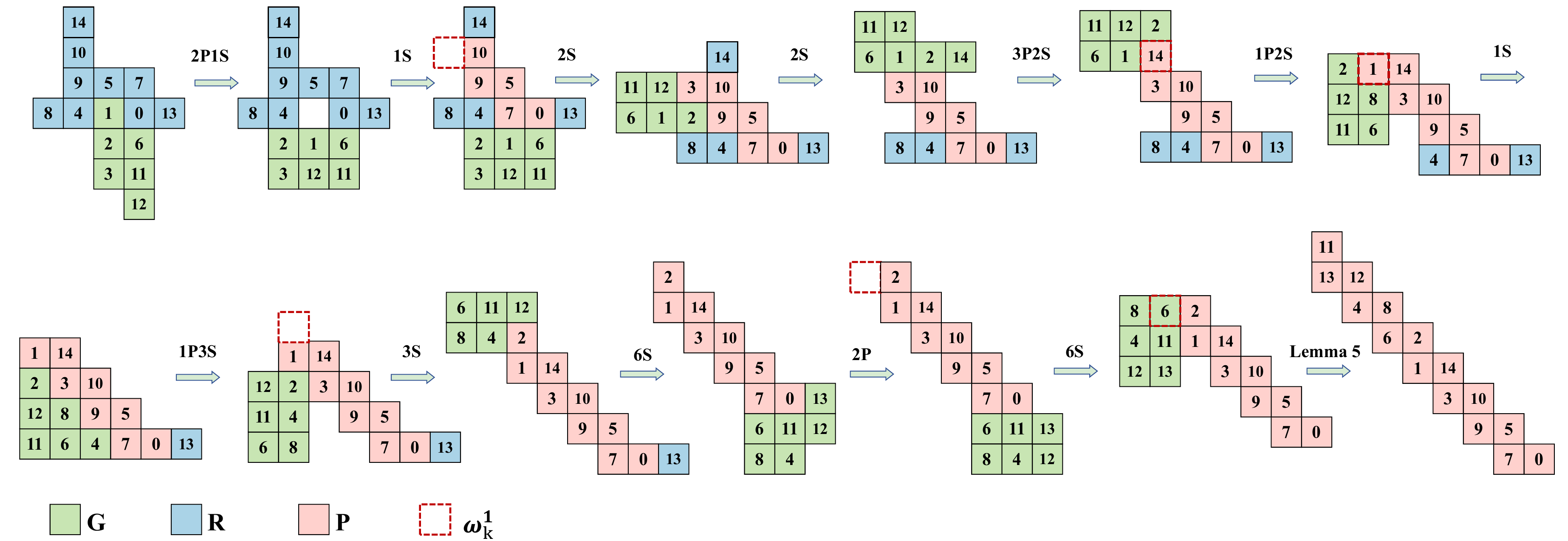}
    \caption{Reconfiguration demonstration for a 15-module example. The sequence shows the planned boundary transport process from the initial configuration to the target staircase configuration, where green, blue, and pink cells denote the gadget $G$, remainder $R$, and staircase prefix $P$, respectively. The dashed red cell marks the first staircase site $\omega_k^1$ in the lookahead window, and the arrows indicate the executed primitive motions.}
    \label{fig:example}
\end{figure*}

\begin{figure}[t]
    \centering
    \includegraphics[width=\linewidth]{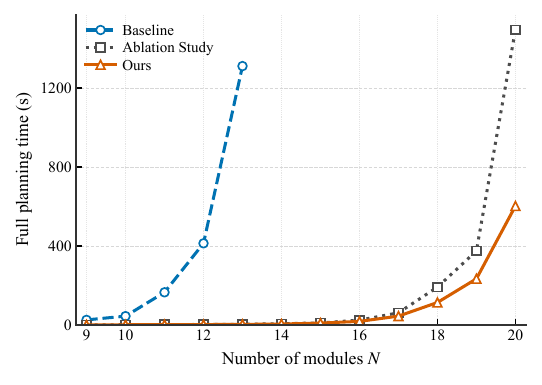}
    \caption{Full-reconfiguration planning-time comparison with respect to $N$. Baseline denotes the prior initial-to-target planner. Ablation Study and Ours denote the proposed two-stage protocol through $S_N$, without and with the lookahead selector, respectively.}
    \label{fig:statistics_ablation}
\end{figure}
The planner then chooses the minimizer in \eqref{eq:unified-lookahead-selection}, executes one corresponding iteration of Algorithm~\ref{alg:canonicalize} with $x(q_k^\star)$, and takes $z_k^\star$ as the selected post-capture canonical state. Only this iteration is committed; if reconfiguration continues, the window and candidate set are rebuilt after the prefix is extended.

Completeness is preserved because the selector only reorders admissible baseline choices. Every harvested module lies in $A_k$, every selected state in $\mathcal Z_k(q)$ is reachable by the same local primitive search used by \textsc{GadgetNormalize}, and the fallback rule recovers the baseline choice when the ranked set is empty. Therefore each executed cycle remains a valid cycle of Algorithm~\ref{alg:canonicalize}, and the termination argument of Theorem~\ref{thm:reachability-staircase} is unchanged.

\section{Experiments}
\label{sec:experiments}
The experiments evaluate the proposed planner through a qualitative reconfiguration example and full-reconfiguration runtime studies. The runtime studies compare with the prior framework under an initial-to-target protocol and use an ablation study to evaluate the contribution of the lookahead selector.

\subsection{Reconfiguration Demonstration}

Fig.~\ref{fig:example} provides a qualitative demonstration of the complete planner by visualizing the transport process that reconfigures a 15-module instance from a fixed initial configuration to the target staircase. By showing the entire sequence, this case gives an executable illustration of the constructive proof: under the theorem assumptions, a valid initial configuration is reduced to the canonical staircase by repeatedly applying the operations in Algorithm~\ref{alg:canonicalize}, including \textsc{Pickup}, \textsc{BoundaryFlow}, and \textsc{PlaceAndExtend}.

This example highlights three properties of the planner. First, every intermediate configuration remains connected because the algorithm transports only removable boundary modules. Second, each transport operation can be decomposed into the primitive pivoting and shearing motions defined earlier. Third, the sequence explicitly shows boundary transport by the gadget in its empty and carrying modes, $G^\circ$ and $G^\bullet$.

\subsection{Comparison with Prior Framework}

We compare the proposed planner with the prior framework~\cite{gu2026selfreconfiguration} under an initial-to-target reconfiguration task. The prior framework directly plans from a given initial configuration to a given target configuration. To report the same task-level quantity, our planner forms the full path by canonicalizing the initial configuration to $S_N$ and concatenating the reverse canonicalization path from the target configuration to $S_N$. For each $N\in[9,20]$, we evaluate 40 seeded initial-to-target pairs of random edge-connected, non-straight configurations, with each full reconfiguration consisting of two canonicalization runs through $S_N$.

We report success rate and average full-reconfiguration runtime. For the proposed variants, the reported runtime is the sum of the two canonicalization stages, measured at the implementation level rather than as a planning-kernel-only runtime.

The prior framework reports $100\%$ success on its tested instances, while both proposed variants complete all trials used here. The formal reachability guarantee for our planner follows from Theorem~\ref{thm:reachability-staircase}. Fig.~\ref{fig:statistics_ablation} reports the average initial-to-target planning time. On the shared module counts, the prior framework grows rapidly from $N=9$, whereas the proposed planner maintains lower full-reconfiguration runtime.

The comparison should be interpreted as a reference comparison, rather than a calibrated speedup benchmark, because such a benchmark would require rerunning both planners on identical initial-target pairs with matched hardware, timeout policy, implementation language, and visualization settings.

\subsection{Ablation Study}

The ablation study evaluates the lookahead selector while keeping the constructive planner, primitive set, configuration decomposition procedure, and target staircase definition fixed. The only changed component is the high level selection criterion for choosing the next harvested module and post-capture state.

The ablation study in Fig.~\ref{fig:statistics_ablation} shows that the lookahead selector reduces repeated planning by accounting for the delivery cost after capture. At $N=20$, for example, the full planner reduces full-reconfiguration planning time from $1496.08\,\mathrm{s}$ in the ablation study to $602.60\,\mathrm{s}$, a reduction of approximately $59.7\%$.

\section{Conclusion}
This paper presented a constructive self-reconfiguration framework for planar deformable rhombus modular robots under a square-cell abstraction. We defined admissible pivoting and shearing primitives, proved that every non-straight edge-connected configuration with $N\ge 7$ can be transformed to the canonical staircase $S_N$, and used this result to establish mutual reconfigurability within the abstraction. The proof is realized by a staircase-canonicalization planner that maintains a staircase prefix, a remainder, and a mobile gadget while preserving edge connectivity. We further added a unified boundary-to-delivery lookahead selector that changes only the high level ordering of admissible choices. On the tested random instances, the planner variants achieved $100\%$ success, and the lookahead selector reduced the two-stage full-reconfiguration runtime in the reported ablation study. Future work will focus on larger instances, tighter cost models, obstacle and workspace constraints, and closed-loop validation on physical deformable modular platforms.

\bibliographystyle{IEEEtran}
\bibliography{references}

\end{document}